\documentclass{article} % For LaTeX2e
\usepackage{iclr2027_conference,times}

\usepackage{amsmath,amsfonts,bm}

\def\eqref#1{Equation~\ref{#1}}
\def\1{\bm{1}}

\DeclareMathAlphabet{\mathsfit}{\encodingdefault}{\sfdefault}{m}{sl}
\SetMathAlphabet{\mathsfit}{bold}{\encodingdefault}{\sfdefault}{bx}{n}

\newcommand{\E}{\mathbb{E}}

\usepackage{url}
\usepackage{enumitem}
\usepackage{algorithm}
\usepackage{algpseudocode}
\algrenewcommand{\algorithmiccomment}[1]{\hfill // #1}
\usepackage{amsthm}
\DeclareMathAlphabet{\mathbbold}{U}{bbold}{m}{n} % hollow 1 for indicator variables
\theoremstyle{definition}

\newtheorem{theorem}{Theorem}
\newtheorem{proposition}{Proposition}
\usepackage{tikz}
\usetikzlibrary{decorations.pathreplacing}
\usepackage{wrapfig}
\usepackage{graphicx}
\usepackage{subcaption}
\usepackage{pgf}
\usepackage{DejaVuSans}
\usepackage[breakable]{tcolorbox}
\usepackage{listings}
\usepackage{upquote}

\usepackage{hyperref}
\usepackage{amsmath}
\usepackage{etoolbox}

\newcommand{\del}[1]{\textcolor{gray}{#1}} % deleted prose, one wrapper per paragraph: reject = drop the wrapper
\AtBeginDocument{%
  \NewCommandCopy{\deloriginalref}{\ref}%
  \NewCommandCopy{\deloriginalsection}{\section}%
  \NewCommandCopy{\deloriginalsubsection}{\subsection}%
  \NewCommandCopy{\deloriginalsubsubsection}{\subsubsection}%
}
\newif\ifdelreferences
\newcommand{\vocab}{\mathcal{V}}

\newcommand{\pt}{\tilde{p}}

\title{Speculative Diffusion is a Tale of \\Two Tradeoffs}

\author{
  Aaron Zorel$^\ast$,
  Guy Van den Broeck$^\dagger$,
  Aditya Grover$^\dagger$,
  Daniel Israel$^\dagger$ \\
  $^\ast$Independent Researcher \quad $^\dagger$University of California, Los Angeles
}
\iclrfinalcopy % Uncomment for camera-ready version, but NOT for submission.

\usepackage[T1]{fontenc}
\NewDocumentCommand{\codex}{+m}{%
  \begingroup\color{blue}%
  #1%
  \endgroup}
\NewDocumentEnvironment{codextheorembox}{+b}{%
  \par\addvspace{10pt}\noindent
  \begingroup
  \setlength{\fboxsep}{8pt}%
  \setlength{\fboxrule}{0.5pt}%
  \codex{\fcolorbox{blue!30!gray}{blue!2}{%
    \begin{minipage}{\dimexpr\linewidth-2\fboxsep-2\fboxrule\relax}
      \color{black}#1%
    \end{minipage}}}%
  \endgroup\par\addvspace{10pt}%
}{}
\begin{document}

\maketitle
\begin{abstract}
Speculative decoding is a technique that accelerates the inference speed of autoregressive large language models (LLMs) by sampling from a faster, lower quality draft model.
Traditionally, a fundamental characteristic of speculative decoding is that it is lossless with respect to the verifier distribution, guaranteeing speed without sacrificing quality.
This differs from recently developed diffusion LLMs (dLLMs), which are notably able to achieve substantial speedups by degrading the quality of their outputs with less diffusion effort.
Because diffusion models are able to sample faster than autoregressive models by generating tokens in parallel, they are considered strong candidates for draft models in speculative diffusion.
While it is widely understood that diffusion models can sacrifice quality for speed, speculative decoding itself can also be generalized to generate lossy samples faster using more lenient acceptance criteria with respect to the verifier distribution.
Presented with the two parameters of diffusion effort and speculative decoding leniency, 
this paper explores the full achievable Pareto frontier via joint variations of these parameters. 
Our experiments show that jointly varying diffusion effort and acceptance leniency can extend the speed–quality Pareto frontier over single step diffusion drafting.
\end{abstract}

\section{Introduction}

% Codex proposed deletion; remove this delblock pair to restore the prior source.

Autoregressive (AR) large language models (LLMs) generate its output token by token, as the distribution of the next token over the vocabulary depends on all the previously generated tokens~\citep{vaswani2017attention}.
Because each sequential decoding step requires transferring all model parameters and the growing key-value (KV) cache for a single token, inference is severely memory-bandwidth bound ~\citep{yuan2024llm}.
Thus, sequential dependency is a major bottleneck on the inference speed of AR decoding~\citep{stern2018blockwise,leviathan2023fast}.

In light of this problem, speculative decoding was invented to speed up the inference of AR LLMs while preserving the original output distribution~\citep{leviathan2023fast,chen2023accelerating}. It employs a much faster model as a drafter~\citep{xia2024unlocking} to generate a short candidate sequence during each drafting round. The original LLM acts as a verifier and processes these proposed tokens in parallel via rejection sampling, accepting an initial prefix while strictly guaranteeing that the sampled distribution matches that of the verifier~\citep{leviathan2023fast,chen2023accelerating}.
It is known that the rejection probability at a given position is the total variation (TV) distance between the drafter distribution and the verifier distribution at that position~\citep{leviathan2023fast,chen2023accelerating,yin2024theoretical}. Therefore, a drafter more aligned with the verifier will have a higher expected acceptance rate across the sequence and a longer accepted prefix in each drafting round. Because only one verifier call is made during each drafting round, a longer accept length will lead to a higher throughput in output tokens per unit time. Because the resulting sequence is provably drawn from the verifier model's output distribution, traditional speculative decoding is \textit{lossless}, providing an exact distributional guarantee that yields acceleration without degrading generation quality.

\begin{figure}[t]
\centering
\codex{\includegraphics[width=\linewidth]{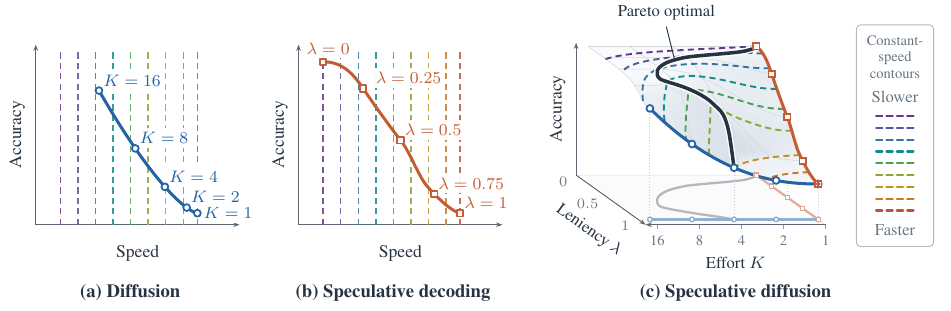}}
\caption{\textbf{A conceptual overview.} Diffusion has a speed-quality tradeoff parameterized by diffusion effort $K$, while speculative decoding has a speed-quality tradeoff parameterized by acceptance leniency $\lambda$. This work argues that speculative diffusion exposes both tradeoffs and therefore reveals optimal configurations that could not be found by varying only a single parameter.}
\label{fig:two-tradeoffs}
\end{figure}

While this exact distributional guarantee ensures that model quality is fully preserved, lossless decoding limits the maximum achievable speedup by how closely the drafter can approximate the verifier. Various methods explore \textit{lossy} speculative decoding, deliberately relaxing the exact distributional guarantee to trade off output quality for faster inference. For example, \citet{leviathan2023fast} experimented with a lenience parameter that increases the acceptance rate of draft tokens by a predefined ratio. 
Cactus~\citep{hao2026cactus} formalizes a relaxation by constraining the distributional divergence (such as Kullback-Leibler divergence) between the drafter and verifier to remain beneath a specified threshold. These lossy formulations provide a mechanism to control the trade-off between output quality and speed.

In parallel with developments in speculative decoding, diffusion large language models (dLLMs) have emerged as an alternative to autoregressive models. Rather than decoding from left to right, dLLMs iteratively denoise a noise corrupted sequence over multiple steps. Diffusion language models can be categorized by their noising process into masked diffusion~\citep{sahoo2024simple,nie2025large,ye2025dream} and uniform diffusion~\citep{austin2021structured,sahoo2025diffusion,diffusiongemma2026}. By updating all token positions in parallel during each denoising step, diffusion models avoid the sequentiality demanded by autoregressive decoding. Crucially, diffusion language models possess an intrinsic speed-quality trade-off: decreasing the number of denoising steps increases speed at the cost of quality.

More recently, an increasing number of diffusion LLMs have been used in speculative decoding. Due to their ability to generate multiple tokens simultaneously, diffusion LLMs serve as natural candidates as draft models. We refer to speculative decoding with a diffusion drafter as \textit{speculative diffusion}~\citep{christopher-etal-2025-speculative,liu2026tidar,chen2026dflash,fu2026nemotron, sandler2026specdiff}.  Speculative diffusion presents two intertwined trade-offs: the drafter trades draft quality for speed via the number of denoising steps, while the verifier can trade quality for speed via lossy acceptance. Mapping out this joint optimization space reveals how both parameters interact to yield the final Pareto frontier between generation speed and output quality (Figure~\ref{fig:two-tradeoffs}).

This paper makes the following contributions to speculative diffusion:

\begin{enumerate}[label=(\roman*), widest=ii, itemsep=2pt, topsep=2pt]
  \item \textbf{Theoretical.} We prove that, for a fixed proposal distribution, linear interpolation between the verifier and draft distributions maximizes token acceptance under a total variation bound on deviation from the verifier (Theorem~\ref{thm:tv-optimal}). We also establish conditions under which rejection sampling preserves the verifier distribution for multistep diffusion drafts (Proposition~\ref{prop:distribution-preserving}). We then provide a theoretical throughput model that explains the interaction between speed and quality in speculative diffusion (Section~\ref{sec:throughput}).
  
  \item \textbf{Empirical.} Across masked and uniform diffusion drafters on reasoning and coding tasks, we show that different combinations of denoising effort and acceptance leniency contribute to the speed-accuracy Pareto frontier. In particular, multistep denoising under more lenient acceptance enables faster generation at comparable accuracy (Section~\ref{sec:experiments}). We show that joint variation over diffusion effort and acceptance leniency often extends the Pareto frontier of speculative diffusion.
\end{enumerate}

\section{Background}
The contributions of this work depend on diffusion language models and lossy speculative decoding. We first provide background on these methods.
\label{sec:prelim}

% \subsection{Diffusion Language Models}
% \label{sec:prelim-dlm}
\paragraph{Diffusion Language Models.}
\label{sec:prelim-dlm}
Diffusion language models iteratively generate text from noise. For a token vocabulary $\vocab$ and sequence of $B$ tokens $\mathbf{x}\in\vocab^B$, let $\mathbf{z}_t$ be its corrupted version at noise level $t\in[0,1]$. The forward noising process independently sets each $z_t^i$ to a noisy token with probability $t$. Masked diffusion sets a special mask symbol $z_t^i=\mathtt{[MASK]}$~\citep{austin2021structured,sahoo2024simple, shi2024simplified}, while uniform diffusion samples $z_t^i\sim\mathrm{Uniform}(\vocab)$~\citep{sahoo2025diffusion,von2025generalized}. The goal of diffusion language models is to learn a denoiser $f_\theta$ that can reverse the forward noise process. In this work, we generate one block of $B$ tokens at a time~\citep{arriola2025block}, conditioned on a fixed prefix $h$. The denoiser $f_\theta(\mathbf{z}_t\mid h)$ outputs a probability distribution over $|\vocab|$ over clean tokens at each block position in parallel. To generate a block, a dLLM starts from noise and iteratively denoises the block until termination. We refer to the number of diffusion steps $K$ as \textit{diffusion effort}. Similar to ``reasoning
effort'' in reasoning models~\citep{guo2025deepseek}, greater effort applies more computation to generate the text,
typically improving quality at the cost of inference speed.

% \subsection{Lossless Speculative Decoding}
% \label{sec:prelim-sd}
\paragraph{Lossless Speculative Decoding.}
\label{sec:prelim-sd}
Speculative decoding is a rejection sampling method to draw samples from a distribution $p$ using a proposal distribution $q$. Let $p(v\mid h)$ and $q(v\mid h)$ denote the verifier and draft distributions over a vocabulary $\vocab$, given a prefix $h$. Speculative decoding accepts a proposal $y\sim q(\cdot\mid h)$ with probability $\min\{1,p(y\mid h)/q(y\mid h)\}$, and upon rejection, samples a replacement with probability proportional to $[p(v\mid h)-q(v\mid h)]_+$, where $[a]_+=\max\{a,0\}$. This rule preserves the verifier distribution~\citep{leviathan2023fast,chen2023accelerating}. Let $\alpha(h)$ denote the probability of accepting a proposal given accepted prefix $h$. Under lossless verification,

\begin{equation}
\label{eq:codex-overlap}
\alpha(h)=\sum_{v\in\vocab}\min\{p(v\mid h),q(v\mid h)\}=1-D_{\mathrm{TV}}(p(\cdot\mid h),q(\cdot\mid h)).
\end{equation}

where $D_{\mathrm{TV}}$ defines the total variation distance between $p$ and $q$.
\begin{equation}
\label{eq:codex-tv}
D_{\mathrm{TV}}(p,q)=\frac12\sum_{v\in\vocab}|p(v)-q(v)|.
\end{equation}
Thus, the acceptance rate of lossless speculative decoding is determined by the alignment of the drafter $q$ and verifier $p$. In practice, the space of possible prefixes is too large to estimate $\alpha(h)$ for every prefix $h$. We therefore measure the expected acceptance rate over a distribution of prefixes $\mathcal{H}$. 
\begin{equation}
\label{eq:expected-acceptance}
\bar\alpha=\mathbb{E}_{h\sim \mathcal{H}}[\alpha(h)],
\end{equation}
which can be estimated with sampling.

% \subsection{Lossy Speculative Decoding}
% \label{sec:lossy}
\paragraph{Lossy Speculative Decoding.}
\label{sec:lossy}
Lossy speculative decoding relaxes the requirement of sampling exactly from $p$ to increase acceptance rate. The degree of relaxation is the \textit{acceptance leniency}, as introduced by \citet[Appendix~A.5]{leviathan2023fast}.  In particular, we consider linear interpolation between the verifier and draft distributions~\citep{fu2025collaborative,hao2026cactus,wang2026diversed} and characterize its optimality under a total variation constraint.

Suppose the target $p$ can be replaced by a modified target $\pt$ in the rejection-sampling rule of Section~\ref{sec:prelim-sd}. If $\pt(\cdot\mid h)$ is chosen before drawing $y\sim q(\cdot\mid h)$, rejection sampling can be performed using $\pt(\cdot\mid h)$ with an acceptance rate of $\alpha(h)=1-D_{\mathrm{TV}}(\pt(\cdot\mid h),q(\cdot\mid h))$. Observe that the distance between $\pt$ and $q$ determines acceptance, while the distance between $\pt$ and $p$ corresponds to the deviation from the original output distribution $p$. 
\begin{samepage}
Given a bound $\epsilon\geq0$, the best achievable acceptance can be characterized by a constrained optimization
\begin{equation}\label{eq:codex-tv-optimization}
\max_{\pt(\cdot\mid h)\in\Delta(\vocab)} 1 - D_{\mathrm{TV}}(\pt(\cdot\mid h),\, q(\cdot\mid h)) \quad \text{subject to} \quad D_{\mathrm{TV}}(\pt(\cdot\mid h),\, p(\cdot\mid h)) \leq \epsilon
\end{equation}
\end{samepage}

where $\Delta(\vocab)$ is the set of probability distributions on the vocabulary.

\begin{codextheorembox}
\begin{theorem}[Optimality of Linear Acceptance]
\label{thm:tv-optimal}
For a fixed prefix $h$ with $p(\cdot\mid h)\ne q(\cdot\mid h)$ and a TV budget $0\leq\epsilon\leq D_{\mathrm{TV}}(p(\cdot\mid h),q(\cdot\mid h))$, an optimal target in \eqref{eq:codex-tv-optimization} is
\begin{equation}
\label{eq:tv-optimal-mixture}
\pt^\star(v\mid h)=(1-\lambda)p(v\mid h)+\lambda q(v\mid h),
\qquad
\lambda=\frac{\epsilon}{D_{\mathrm{TV}}(p(\cdot\mid h),q(\cdot\mid h))}
\end{equation}
and achieves maximum acceptance probability $\alpha^\star(h)=1-D_{\mathrm{TV}}(p(\cdot\mid h),q(\cdot\mid h))+\epsilon$.
\end{theorem}
\end{codextheorembox}

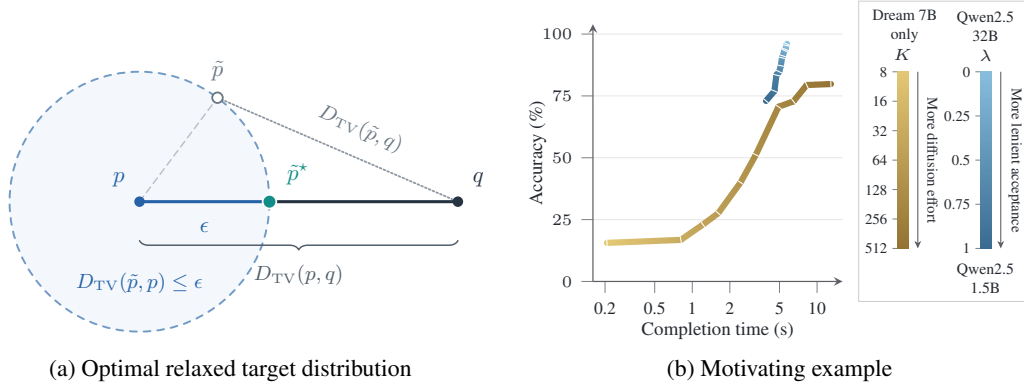
\begin{figure}[t]
\centering
\captionsetup[subfigure]{font=small,justification=centering}
% Keep the replaced placeholder and caption in the source for review.
\renewcommand{\del}[1]{}%
\begin{subfigure}[b]{0.48\linewidth}
\centering
\codex{\resizebox{\linewidth}{!}{%% Schematic for Theorem \ref{thm:tv-optimal}; include inside \codex.
%% p is the verifier, q the draft, and the TV budget is centered at p.
%% The circle is an abstract constraint-set glyph, not a Euclidean embedding
%% of a TV ball. No uniqueness is claimed for the displayed optimum.
%% The displayed case is 0 < epsilon < D_TV(p,q).
%% Bounding-box aspect ratio: 4.35 / 7.45, shared with the placeholder panel.
\begingroup
\definecolor{TVBlue}{HTML}{2563A6}
\definecolor{TVInk}{HTML}{263442}
\definecolor{TVMuted}{HTML}{697683}
\definecolor{TVOptimal}{HTML}{138D88}
\begin{tikzpicture}[
  font=\fontsize{10}{12}\selectfont,text=TVInk,
  line cap=round,line join=round,
]
  \coordinate (tv-p) at (0,0);
  \coordinate (tv-q) at (4.90,0);
  \coordinate (tv-optimal) at (2,0);
  \coordinate (tv-other) at (1.2,1.6);

  \filldraw[fill=TVBlue!5,draw=TVBlue!65,dashed,line width=0.85pt]
    (tv-p) circle (2);

  % Compare a feasible target with the mixture on the p--q segment.
  \draw[TVMuted!45,densely dashed,line width=0.65pt] (tv-p) -- (tv-other);
  \draw[TVMuted!75,densely dotted,line width=0.8pt]
    (tv-other) -- (tv-q)
    node[pos=0.55,sloped,above=4pt,text=TVMuted,font=\fontsize{9}{11}\selectfont]
      {$D_{\mathrm{TV}}(\tilde p,q)$};
  \draw[TVBlue,line width=1.35pt]
    (tv-p) -- (tv-optimal)
    node[midway,below=5pt,text=TVBlue] {$\epsilon$};
  \draw[TVInk,line width=1.35pt]
    (tv-optimal) -- (tv-q);
  \draw[TVInk!80,line width=0.65pt,decorate,
    decoration={brace,mirror,amplitude=4pt}]
    (0,-0.60) -- (4.90,-0.60)
    node[midway,below=7pt,font=\fontsize{9}{11}\selectfont]
      {$D_{\mathrm{TV}}(p,q)$};

  \fill[TVBlue] (tv-p) circle (2.3pt)
    node[above left=3pt,text=TVBlue] {$p$};
  \fill[TVInk] (tv-q) circle (2.3pt)
    node[above right=3pt] {$q$};
  \filldraw[fill=TVOptimal,draw=white,line width=0.75pt]
    (tv-optimal) circle (3pt)
    node[above right=4pt,text=TVOptimal] {$\tilde p^\star$};
  \filldraw[fill=white,draw=TVMuted,line width=0.8pt]
    (tv-other) circle (2.3pt)
    node[above=4pt,text=TVMuted] {$\tilde p$};

  \node[text=TVBlue,font=\fontsize{9}{11}\selectfont] at (0,-1.25)
    {$D_{\mathrm{TV}}(\tilde p,p)\le\epsilon$};
  \pgfresetboundingbox
  \path[use as bounding box] (-2.15,-2.15) rectangle (5.30,2.20);
\end{tikzpicture}%
\endgroup}}
\caption{Optimal relaxed target distribution}
\label{fig:tv-ball}
\end{subfigure}\hfill
\begin{subfigure}[b]{0.48\linewidth}
\centering
\includegraphics[width=\linewidth]{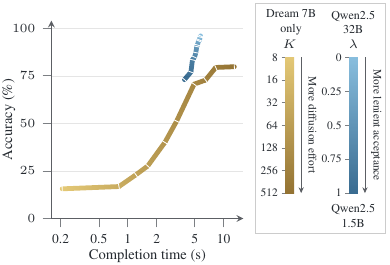}
\caption{Motivating example}
\label{fig:ar_vs_dream}
\end{subfigure}

\caption{
\textbf{Illustrative examples.}
Left: Visualization of the proof of Theorem~\ref{thm:tv-optimal} shows that the linear interpolated distribution $\pt^\star$ is the optimal solution to maximize drafter $q$ and verifier $p$ alignment while staying within a constraint on deviation from the verifier $\epsilon$. The solution to this constrained optimization is $\pt$.
Right: A motivating example shows that the speed-quality frontiers independently induced by diffusion effort and acceptance lenience in speculative decoding can occupy separate regions of the Pareto frontier. We compare diffusion only with AR speculative decoding with leniency.}
\label{fig:tv-motivation}
\end{figure}

In the above theorem, $\lambda$ dictates the acceptance leniency by shifting the target toward the drafter. Lossless sampling from the verifier occurs at $\lambda=0$, while $\lambda=1$ samples from the draft distribution. Figure~\ref{fig:tv-ball} gives an intuitive visual proof of the theorem. Total variation distance is a distance metric, so geometrically the optimal path between two distributions will be a linear interpolation. Refer to Appendix~\ref{app:tv-ball-proof} or ~\ref{app:tv-ball-alt-proof} for the proof.

Note that this result characterizes the best acceptance obtainable by choosing a target distribution \textit{before} observing the proposed token. 
Another point of clarification in the above proof is that it is a statement on the next token distribution, which assumes conditioning on prefix $h$. These bounds are not trivially extendable to the joint distribution of $p$ and $q$ over multiple token positions.

Cactus~\citep{hao2026cactus} is another example that allows lossy acceptance. Unlike linear acceptance, Cactus controls leniency by modifying the verifier distribution \textit{after} observing a proposed token. It increases the probability of that token subject to a bound on the modified target's KL divergence, or any other $f$-divergence, from the verifier, with larger budgets permitting more lenient acceptance. This dependence on the sampled token distinguishes Cactus from the linear acceptance defined above, whose target is chosen before sampling.

\section{Speculative Diffusion}
\label{sec:method}
\label{sec:tradeoffs}

Because diffusion models can draft tokens in parallel, speculative diffusion is a promising approach to achieve fast and high quality LLM inference. As discussed previously, both diffusion and speculative decoding have inherent speed quality tradeoffs.

\subsection{A Motivating Example}
\label{sec:motivating-examples}

To motivate varying diffusion denoising effort and speculative decoding leniency jointly, it is useful to examine each parameter independently. This work is motivated by the following observations.

First, even with a large leniency under lossy speculative decoding and a very small AR drafter, the range of achievable drafting speeds is much narrower than that of in a dLLM when varying the diffusion effort.
For example, Figure~\ref{fig:ar_vs_dream} shows that ignoring the verifier and completely relying on the AR drafter distribution Qwen2.5-1.5B-Instruct~\citep{qwen2024qwen25} achieves a speedup, but does cannot capture the full frontier. Meanwhile, a diffusion model is capable of can running over 10$\times$ faster. This suggests dLLM drafters expose a new axis of variation on its drafting quality and speed in a way that traditional AR drafters cannot.

Second, even though varying diffusion effort with the dLLM alone can yield a speed and quality on the Pareto frontier, most dLLMs have significantly lower accuracies than those of the state-of-the-art AR LLMs. For example, in Figure~\ref{fig:ar_vs_dream} the accuracy of Dream 7B plateaus despite running a high number of diffusion steps. Therefore, a way to uplift the accuracy of the diffusion LLMs with an AR model as a guide would be desirable. Many recent works echo this sentiment~\citep{liu2025discrete, israel2026accelerating, hu2026flashdlm}.

With these observations, it is evident that the speed quality tradeoffs of diffusion and lossy speculative decoding are complementary. Diffusion alone can achieve higher speeds. Speculative decoding alone can achieve higher quality. Because each method is Pareto optimal independently, it is possible that combining the approaches can unlock further improvements in the Pareto frontier of speed and quality.

\subsection{Multi-Step Diffusion Drafting}
\label{sec:multi-step-diffusion}

Many instances of speculative diffusion focus on single step diffusion drafting \cite{fu2026nemotron,sahoo2026unlocking, chen2026dflash, cheng2026dspark}. In this work, we use a diffusion drafter for multiple steps per verifier call, as shown in Algorithm~\ref{alg:vgdllm}. While diffusion effort can be parameterized directly in the number of steps $K$ (as mentioned in Section~\ref{sec:prelim-dlm}), other criteria can also be used as parametrization to control the diffusion effort. For example, confidence thresholds \citep{nie2026large} or entropy bounds~\citep{ben2026accelerated} can indirectly affect the number of diffusion steps.

\begin{algorithm}[!t]
\caption{Speculative diffusion}
\label{alg:vgdllm}
\algtext*{EndFor}
\algrenewcommand{\algorithmicindent}{2em}
\algdef{SE}[DOWHILE]{Do}{DoWhile}{\algorithmicdo}[1]{\algorithmicwhile\ #1}
\begin{algorithmic}[1]
\Require prompt $h$; block size $B$; denoising steps $K$;
\Statex denoiser\,$f_\theta$; verifier $p$; target rule \textsc{RelaxTarget}
\While{$\texttt{EOS}\notin h$}
  \State $(y_{1:B},q_{1:B})\gets\Call{DiffusionDraft}{\,f_\theta,h,B,K}$
  \State $p_{1:B+1}\gets\Call{VerifierForward}{p,h,y_{1:B}}$
  % Finish the preceding State before the deletion markup to preserve indentation.
  \csname ALG@closeloops\endcsname
  \State $\pt_{1:B}\gets\Call{RelaxTarget}{p_{1:B},q_{1:B},y_{1:B}}$
  % Finish the preceding State before the deletion markup to preserve indentation.
  \csname ALG@closeloops\endcsname
  \State $k\gets 0$
  \Do
    \State $k\gets k+1$
    \State $u\sim\mathcal{U}(0,1)$
  \DoWhile{$k\leq B$ \textbf{and} $u\leq\pt_k(y_k)/q_k(y_k)$}
  \State $\displaystyle y_k\sim
    \begin{cases}
      \Call{Normalize}{[\pt_k-q_k]_+},&k\leq B,\\[2pt]
      p_k,&k=B+1.
    \end{cases}$
  \State $h\gets \Call{Append}{h,y_{1:k}}$
\EndWhile
\State \Return $\Call{truncate}{h}$
\end{algorithmic}

\end{algorithm}

In speculative decoding with an autoregressive drafter, when sampling token at position $i$ it is conditioned on tokens prior to position $i$ that have been accepted by the verifier. Multi-step diffusion drafting does not maintain this property. A final sampled token can be obtained through many possible denoising trajectories. Evaluating the marginal probability at a position would require summing the probability over an intractable number of trajectories or approximating via an evidence lower bound (ELBO)~\citep{zhao2025unified}. Therefore, it is not immediately clear what proposal probabilities should be used in rejection sampling. 

Since the diffusion trajectory is stochastic, the proposal distribution itself can be seen as a random vector over the probability simplex $\Delta(\vocab)$. Proposition~\ref{prop:distribution-preserving} shows that without knowing the underlying distribution of the proposal distribution random vector, rejection sampling can still be performed on the proposal distribution that generated each token, maintaining exact samples from the verifier. In the context of speculative diffusion, this means the final step probabilities (e.g. final transition probabilities in uniform diffusion or marginals in masked diffusion) can be used as the proposal probabilities in rejection sampling. 

\subsection{A Simple Throughput Model}
\label{sec:throughput}

\begin{figure}[!t]
\centering
% Preserve the replaced four-panel source and caption for review.
\renewcommand{\del}[1]{}
\resizebox{\linewidth}{!}{\input{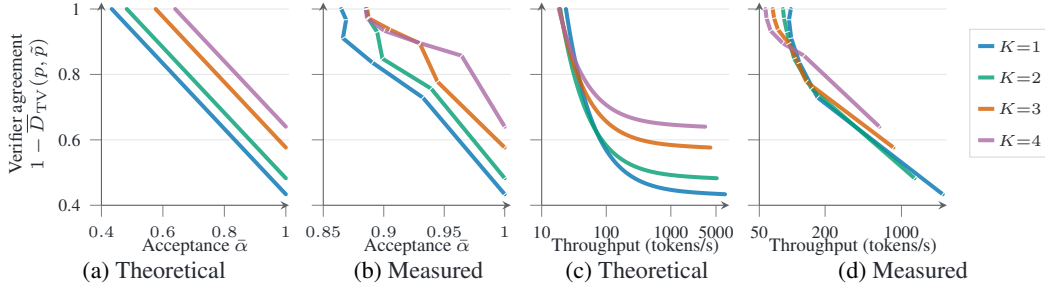}}
\captionsetup{font=small,skip=4pt}
\caption{\textbf{A model of speculative diffusion.} Theoretical and measured tradeoffs for Dream-7B drafting for Qwen2.5-32B on GSM8K. (a)
The theoretical acceptance rate (speed) and verifier agreement (quality) under linear acceptance for each of $K$ diffusion steps per round. (b) Empirically measured acceptance rate for each $K$. (c) A transformation from acceptance rate to throughput according to a simple model of throughput. (d) Empirically measured throughput.}
\label{fig:gsm8k_scaling_law}
\end{figure}

A theoretical model can aid our understanding in how speed and generation quality vary as a function of the diffusion effort and leniency. To this end, the expected acceptance rate $\bar\alpha = 1-D_{\mathrm{TV}}(\pt, q)$ can be used to model throughput and the similarity between the verifier distribution and the modified target distribution $1-D_{\mathrm{TV}}(\pt, p)$ can be used as a proxy for generation quality, where both distances can be derived from an empirically profiled $D_{\mathrm{TV}}(p, q)$ under Theorem~\ref{thm:tv-optimal} at different effort levels. For diffusion effort using $K$ steps per drafting round, let $\mathcal{H}_K$ be a fixed distribution of prefixes and $q_K$ the corresponding draft distribution. Averaging over the distribution of prefixes, we define $d_K=\mathbb{E}_{h\sim \mathcal{H}_K}[D_{\mathrm{TV}}(p(\cdot\mid h),q_K(\cdot\mid h))]$, which is the TV distance between the verifier and drafter at effort $K$ with zero leniency. Applying the linear acceptance rule in Theorem~\ref{thm:tv-optimal} over $H_K$ gives the expected acceptance rate and verifier distance as
\begin{equation}
\label{eq:theoretical_speed_quality}
\bar\alpha_K(\lambda)=1-(1-\lambda)d_K 
\qquad \text{ and } \qquad
1-\overline{D}_{\mathrm{TV}}(p,\pt)=1-\lambda d_K.
\end{equation}

Let $c_d(K)$ and $c_v$ denote the time it takes to call the drafter and verifier per drafting round respectively. Empirically, $c_d(K)\approx a+bK$ and $c_v$ is constant.
Proposition~\ref{prop:expected-tokens-per-round} in Appendix~\ref{app:expected-tokens-per-round} then gives the theoretical throughput as
\begin{equation}
\label{eq:throughput_model}
\mathcal{T}(K,\bar\alpha)=
\frac{\sum_{j=0}^{B}\bar\alpha^{\,j}}{c_d(K)+c_v}.
\end{equation}

Figures~\ref{fig:tvd-measured-acceptance} and~\ref{fig:tvd-measured-throughput} directly measure each quantity predicted in Figures~\ref{fig:tvd-theoretical-acceptance} and~\ref{fig:tvd-theoretical-throughput}, and the details for obtaining these results can be found in Appendix~\ref{app:cost-calibration}. The theoretical model captures two important qualitative features seen in the empirically measured results: (1) The tradeoff between acceptance rate and verifier agreement is approximately linear and (2) no denoising budget captures the full throughput-quality Pareto frontier. The quantitative discrepancy between theoretical and measured results is driven by multiple factors. First, the assumption of a universal acceptance rate $\bar\alpha$ is incorrect. Leniency changes which draft tokens are accepted and hence the prefixes on which subsequent predictions are conditioned. More formally, the prefix distribution $\mathcal{H}_K$ encountered during the empirical profiling differs from the prefix distribution at higher leniency, resulting in a different value of $d_K$. Another difference between theory and practice is that the theory computes acceptance over an entire block, but in practice throughput can be affected by EOS termination effectively reducing drafting block size.

This simple setup illustrates the qualitative features of speculative diffusion when varying diffusion effort and acceptance leniency. However, it is not obvious how quality measured by total variation translates to downstream metrics such as reasoning accuracy. Furthermore, the linear interpolation acceptance criterion enables simpler modeling of acceptance rate, but other acceptance criteria may not behave the same way. These questions will be addressed in the following experimental section.

\section{Experiments}
\label{sec:experiments}
\begin{figure}[!t]
\centering
\codex{\includegraphics[width=\linewidth]{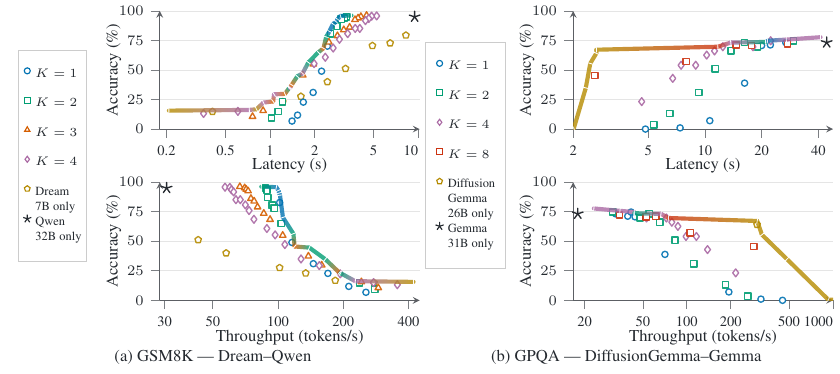}}
\captionsetup{font=small,skip=4pt}
\caption{\textbf{Examples of the Pareto frontier.} We provide examples of speculative diffusion for both masked diffusion (left) and uniform diffusion (right). Each point signifies a configuration of acceptance leniency in verification and diffusion effort ($K$ steps) in drafting. We observe the same qualitative trend that the speed-quality Pareto frontier is occupied by many different settings of diffusion effort which require different levels of lenience.}
\label{fig:joint-frontier}
\end{figure}
The goal of our experiments is to empirically explore the two tradeoffs of speculative diffusion: diffusion effort and leniency. We first establish a common experimental setup, then conduct experiments in two different settings.
\subsection{Experimental Setup
\label{sec:experimental-setup}}

We first examine masked and uniform diffusion using the following model pairs: Dream-7B~\citep{ye2025dream} with Qwen2.5-32B and DiffusionGemma-26B-A4B~\citep{diffusiongemma2026} with Gemma4-31B~\citep{gemmateam2026gemma4}. We then compare several other recent speculative diffusion models under a common Qwen3-8B verifier. 

In general, we benchmark quality using GSM8K~\citep{cobbe2021training}, GPQA Diamond~\citep{rein2024gpqa}, and HumanEval~\citep{chen2021evaluating} for math, reasoning, and code. In our Dream-Qwen and DiffusionGemma-Gemma experiments, some model-benchmark combinations are excluded based on two criteria that depend on the standalone capabilities of the drafter and verifier. The first criterion is that if the standalone accuracies of the drafter and verifier models are nearly the same on a benchmark, then there is no benefit to running speculative decoding, and thus the combination will be excluded. The second criterion is that if the drafter model performs worse than random, then the combination will be excluded.

Table~\ref{tab:drafter-verifier-gaps} reports the baseline accuracies for both the standalone drafter and verifier of the Dream, SDAR, and DiffusionGemma pairs, on GSM8K, HumanEval, and GPQA Diamond. Based on these measurements, we evaluate Dream and SDAR on GSM8K and HumanEval, and DiffusionGemma on GPQA. An evaluation configuration consists of a benchmark, model pair, sampling configuration, diffusion effort, and acceptance leniency. The detailed experimental configurations can be found in the Appendix~\ref{app:experimental-config-degtail}.

\subsection{Joint Exploration Over Two Tradeoffs}
\label{sec:joint-exploration-results}

Figure~\ref{fig:joint-frontier} is obtained by fixing a denoising budget $K$ and varying the acceptance leniency. 
In this particular plot, we display the results using the modified Cactus acceptance criterion instead of linear acceptance to demonstrate the generality of our findings.
The corresponding linear-acceptance results are in Figures~\ref{fig:gsm8k_time_cactus_with_ar} and~\ref{fig:gpqa_gemma_time_convex} and are qualitatively similar.
See Appendix~\ref{app:codex-cactus} for details.

Every point in Figure~\ref{fig:joint-frontier} represents the speed and accuracy measured from a diffusion effort, acceptance leniency pair. We then visualize the Pareto frontier over all pairs for latency and throughput. Note that latency is wall-clock time to problem completion, while throughput is the number of generated tokens divided by the wall clock time to problem completion. Both are needed to quantify speed, because throughput can be over optimized by generating unnecessary tokens, yet latency improvements achieved entirely from shorter generation lengths is also not desirable. The Pareto frontier is visualized as a thick connected line, while hollow markers are measurements not on the frontier. For visual clarity, we omit the points on the Pareto frontier itself and instead mark them with white ticks on the line. The Pareto frontier's color identifies the diffusion effort $K$ in that region of the frontier. Between measurements of different colors on the frontier, we apply a color gradient. This makes it possible to see which denoising budgets contribute to the frontier. 

For both model pairs in Figure~\ref{fig:joint-frontier}, several values of $K$ contribute to the frontier. Varying leniency at a single diffusion effort configuration will fail to achieve Pareto-optimal results. In general, low diffusion effort with strict acceptance requirements tends to capture the higher accuracy regions of the Pareto frontier, while higher diffusion effort with more lenient acceptance captures the higher speed regions. This is natural because as previously observed, diffusion has higher speed but lower quality and autoregression has higher quality but lower speed. The observation that the Pareto frontier contains many configurations of diffusion effort and acceptance leniency persists across drafting block sizes. In Appendix~\ref{app:dream-block-size}, we ablate drafting block size and show that no single diffusion effort captures the full frontier.
\begin{figure}[!t]
\centering
\codex{\includegraphics[width=\linewidth]{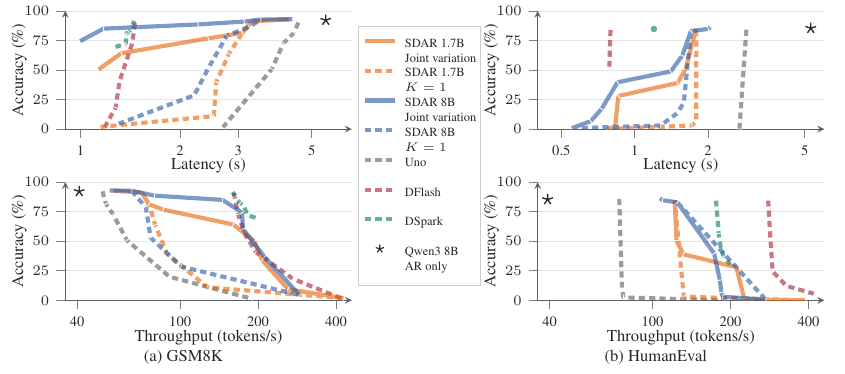}}
\captionsetup{font=small,skip=4pt}
\caption{\textbf{Comparison of Pareto frontiers.} For a single Qwen3 8B verifier, we benchmark several speculative diffusion methods. Each line corresponds to the achievable speed-quality frontier of a specific draft model. Dashed curves fix one diffusion step and vary leniency, while solid curves are varied over diffusion effort and acceptance lenience. For the SDAR models that support multi-step speculative diffusion, varying over the parameters jointly significantly extends the Pareto frontier.}
\label{fig:sdar-uno-frontier}
\end{figure}
\subsection{Comparing Drafters with a Shared Verifier}
\label{sec:shared-verifier-results}

While the previous experimental setup analyzes speculative diffusion for model pairs that were trained separately, many recent works train specifically for diffusion speculation. Under a Qwen3 8B~\citep{yang2025qwen3} verifier, these works include Uno~\citep{sahoo2026unlocking}, DFlash~\citep{chen2026dflash}, and DSpark~\citep{cheng2026dspark}. For this experiment we also include SDAR~\citep{cheng2026sdar}, which notably was not trained for the task of speculative decoding.

We compare SDAR-1.7B, SDAR-8B, Uno, DFlash, and DSpark on GSM8K and HumanEval with a shared Qwen3-8B verifier. For each model, we apply linear acceptance at $\lambda\in\{0,0.2,0.4,0.6,0.8,1\}$, from lossless verification to full acceptance. Uno, DFlash, and DSpark are trained specifically for one step diffusion drafting. DFlash and DSpark apply feature conditioning from the verifier, making multiple diffusion steps per verifier call impossible. Thus, for these models we only include their Pareto frontier using one step drafting and acceptance lenience. For Uno, we evaluate $K\in\{1,2,4\}$ in Appendix~\ref{app:uno-multistep}, but we find only $K=1$ one-step diffusion occupies its Pareto frontier. Because SDAR is trained as a standalone diffusion model with a multistep objective, we measure it's Pareto frontier by varying diffusion steps as well. We vary SDAR's adaptive denoising effort using a confidence threshold $\tau\in\{0.3,0.5,0.7,0.9\}$ and also evaluate fixed one-step ($K=1$) drafts. For each model size, we compare the joint frontier over all these settings with the frontier restricted to $K=1$. Appendix~\ref{app:shared-verifier-sweeps} shows all measured configurations. The results are shown in Figure~\ref{fig:sdar-uno-frontier}. Joint variation over diffusion effort and acceptance leniency substantially improves SDAR's speed-accuracy tradeoff. While the specialized one-step drafters remain strong, especially in throughput, joint variation places SDAR on the frontier despite not having specialized training and architecture for speculative diffusion. These results motivate training more specialized drafters for multi-step refinement.

\section{Related Work}

\paragraph{Diffusion drafters.}
Recent work on diffusion drafting for speculative decoding aims to generate proposals with fewer sequential model evaluations. SpecDiff introduced diffusion drafters and studied how denoising steps balance drafting cost against acceptance~\citep{christopher-etal-2025-speculative}. Subsequent work improves draft-verifier alignment through causal-consistency path search~\citep{li2025diffuspec} or conditioning diffusion drafts on target-model features~\citep{chen2026dflash}, while hybrid models integrate diffusion drafting and autoregressive verification within one model~\citep{liu2026tidar,fu2026nemotron, sahoo2026unlocking}. In general, diffusion drafting in speculative decoding rarely applies multiple steps of diffusion effort because it does not typically yield benefit ~\citep[Appendix~A.2]{sandler2026specdiff}. While quadratic decoding ~\citep{samragh2025your} applies multiple steps of diffusion for the purpose of tree-based verification ~\citep{miao2024specinfer}, tree-based speculative decoding is orthogonal to acceptance leniency and is beyond the scope of this work. 

\paragraph{Lossy speculative decoding.} Multiple works have introduced leniency into speculative decoding acceptance. It was first introduced in the seminal speculative decoding paper by \citet{leviathan2023fast}. Later theoretical work formalized the tradeoff as a linear relationship between rejection rates and total variation loss~\citep{yin2024theoretical}. In this paper, we spotlight the linear acceptance criterion which has been mentioned in several previous works under different names: collaborative decoding~\citep{fu2025collaborative}, linear interpolation~\citep{hao2026cactus}, and dynamic ensemble verification~\citep{wang2026diversed}. While the method is not original, to the best of our knowledge, our formulation is the first analyze it from the perspective of constrained optimization and show that it is optimal under a total variation constraint. Furthermore, these works do not investigate leniency in the context of speculative diffusion.
Nemotron-Labs-Diffusion briefly describes a verifier that interpolates AR and diffusion distributions~\citep[Appendix~C]{fu2026nemotron}, but it is mentioned as a possible direction without experimental evaluation. 

\section{Conclusion}
This work demonstrates that speculative decoding with diffusion drafters, also known as speculative diffusion, is governed by two complementary mechanisms that trade speed for quality. Specifically, these mechanisms are the denoising effort of the drafter and the acceptance leniency of the verifier. Our cost model and empirical findings show that neither parameter independently can capture the full Pareto frontier between generation speed and task accuracy. The speed quality trade-offs of speculative diffusion expand the design space for large language model serving.
A promising future direction is to improve query routing, where incoming requests are mapped to distinct configurations along the Pareto frontier according to their task difficulty and tolerance for error. Beyond routing at the query level, future systems could employ dynamic scheduling policies that adjust denoising budgets and verification thresholds at runtime. 

\clearpage

\subsection*{AI Usage Disclosure}
The authors used LLMs to assist in finding relevant papers, plotting results, polishing LaTeX, and drafting code.
We take responsibility for the final content of this work.
\subsubsection*{Acknowledgments}
This work was funded in part by the DARPA CLARA program under Agreement No. HR0011269E224, NSF CAREER \#2341040, Schmidt Sciences, and gifts from Cisco Research, Qualcomm, and Amazon.

\bibliography{iclr2027_conference}
\bibliographystyle{iclr2027_conference}

\appendix

\section{Proof of Theorem~\ref{thm:tv-optimal}}
\label{app:tv-ball-proof}

\begin{proof}
For any feasible $\pt(\cdot\mid h)$, the triangle inequality gives
\[
\begin{aligned}
D_{\mathrm{TV}}(\pt(\cdot\mid h),q(\cdot\mid h))
&\geq D_{\mathrm{TV}}(p(\cdot\mid h),q(\cdot\mid h))
-D_{\mathrm{TV}}(\pt(\cdot\mid h),p(\cdot\mid h))\\
&\geq D_{\mathrm{TV}}(p(\cdot\mid h),q(\cdot\mid h))-\epsilon.
\end{aligned}
\]
Hence, the acceptance probability is at most
\[
1-D_{\mathrm{TV}}(p(\cdot\mid h),q(\cdot\mid h))+\epsilon.
\]

We must now show that this bound is attained. Consider
\[
\pt^\star(\cdot\mid h)=(1-\lambda)p(\cdot\mid h)+\lambda q(\cdot\mid h),
\qquad
\lambda=\frac{\epsilon}{D_{\mathrm{TV}}(p(\cdot\mid h),q(\cdot\mid h))}.
\]
Under the theorem's assumptions, $\lambda\in[0,1]$. By the definition of total variation distance,
\[
\begin{aligned}
D_{\mathrm{TV}}(\pt^\star(\cdot\mid h),p(\cdot\mid h))
&=\frac12\sum_{v\in\vocab}
\left|\lambda\bigl(q(v\mid h)-p(v\mid h)\bigr)\right|\\
&=\lambda D_{\mathrm{TV}}(p(\cdot\mid h),q(\cdot\mid h))
=\epsilon,
\end{aligned}
\]
so $\pt^\star$ is feasible. Similarly,
\[
\begin{aligned}
D_{\mathrm{TV}}(\pt^\star(\cdot\mid h),q(\cdot\mid h))
&=\frac12\sum_{v\in\vocab}
\left|(1-\lambda)\bigl(p(v\mid h)-q(v\mid h)\bigr)\right|\\
&=(1-\lambda)D_{\mathrm{TV}}(p(\cdot\mid h),q(\cdot\mid h))\\
&=D_{\mathrm{TV}}(p(\cdot\mid h),q(\cdot\mid h))-\epsilon.
\end{aligned}
\]
Thus $\pt^\star$ attains the lower bound and therefore achieves acceptance rate
\[
1-D_{\mathrm{TV}}(p(\cdot\mid h),q(\cdot\mid h))+\epsilon.
\]
\end{proof}

\section{Alternative Proof of Theorem~\ref{thm:tv-optimal}}
\label{app:tv-ball-alt-proof}
This alternative proof helps to provide a more geometric understanding of the result.

Here we abbreviate $p(\cdot\mid h)$ as $p$ and $q(\cdot\mid h)$ as $q$.

View $p$ and $q$ as probability measures on a $\sigma$-algebra $\mathcal{F}$ on a set $\Omega$, and define
\[
  p_r \;:=\; p + r\,(q - p) \quad \text{for } r \in [0,\, 1].
\]

Since the total variation is a norm on the real vector space of finite signed measures,
denote it as $\|\cdot\|_{\mathrm{TV}}$, but we shall prove the statement for any norm $\|\cdot\|$.

\begin{proof}
We show that
\begin{equation}
\label{eq:theorem-alt-proof}
  \|p_r - q\| \;\leq\; \|\pt - q\|
  \quad \text{ any signed measure } \pt \text{ such that } \|\pt - p\| \leq \|p_r - p\|.
\end{equation}

Indeed, by the absolute homogeneity property of norms, we have
\[
  \|p_r - q\| \;=\; \|(1-r)\,(p - q)\| \;=\; (1-r)\,\|p - q\|
\]
\[
  \|p_r - p\| \;=\; \|r(q - p)\| \;=\; r\,\|p - q\|
\]

Now for any $\pt$ such that $\|\pt - p\| \leq \|p_r - p\| = r\,\|q - p\|$,
the triangle inequality implies that
\[
  \|\pt - q\| \;\geq\; \|q - p\| - \|\pt - p\| \;\geq\; \|q - p\| - r\,\|q - p\| \;=\; (1 - r)\,\|q - p\| \;=\; \|p_r - q\|.
\]

In words, this is saying that the signed measure $p_r$ is the closest point to $q$
among all points that are at most $\|p_r - p\|$ distance away from $p$.

Note that ~\eqref{eq:theorem-alt-proof} is true for any signed measure $\pt$,
so in particular it is true for any probability measure $\pt$.

Also note that it is easy to see that $p_r$ is a probability measure,
which means that the probability measure $p_r$ is the closest to $q$ among
all signed measures within the ball, and thus the closest among all
probability measures within the ball.

Now, with the choice of $r$ equal to $\lambda = \frac{\epsilon}{D_{\mathrm{TV}}(p(\cdot\mid h),q(\cdot\mid h))}$, it follows that the optimal solution to \eqref{eq:codex-tv-optimization} is $\pt^\star(v\mid h)=(1-\lambda)p(v\mid h)+\lambda q(v\mid h)$.
\end{proof}

\section{Expected Number of Generated Tokens per Drafting Round}
\label{app:expected-tokens-per-round}

\begin{proposition}
\label{prop:expected-tokens-per-round}
Assume there is a draft sequence of length $B$. Let $\mathbbold{1}_i$ be the indicator variable for whether token $i$ is accepted conditioned on all previous tokens being accepted, where $i \in [B]$.
Assume $\{\mathbbold{1}_i\}$ are independent and let $\mathbb{P}(\mathbbold{1}_i = 1) = P_i$, where $P_i$ is another set of i.i.d.\ random variables.
Let $\mu := \E P_i$.
Then the expected number of generated tokens is
\begin{equation}\label{eq:expected-tokens-per-round}
\sum_{n=0}^{B} \mu^n =
\begin{cases}
\dfrac{1 - \mu^{B+1}}{1 - \mu} & \text{for } \mu \neq 1 , \\[1.5ex]
1 + B & \text{for } \mu = 1 .
\end{cases}
\end{equation}
\end{proposition}

\begin{proof}
The expected number of accepted tokens for a draft sequence of length $B$ is
\begin{equation}\label{eq:expected-accepted}
A(B) = \E\big[\mathbbold{1}_1 + \mathbbold{1}_1 \mathbbold{1}_2 + \dots + \mathbbold{1}_1 \mathbbold{1}_2 \cdots \mathbbold{1}_B\big] = \sum_{n=1}^{B} \E \prod_{i=1}^{n} \mathbbold{1}_i .
\end{equation}

Since the indicator variables $\mathbbold{1}_i$ are independent, we have
\begin{equation}\label{eq:conditional-product}
\E\Big[\prod_{i=1}^{n} \mathbbold{1}_i \,\Big|\, P_1, \dots, P_n \Big] = \mathbb{P}(\mathbbold{1}_i = 1\ \forall i \le n) = \prod_{i=1}^{n} P_i ,
\end{equation}
and
\begin{align}
\E\Big[\prod_{i=1}^{n} \mathbbold{1}_i\Big]
&= \E\Big[\E\Big[\prod_{i=1}^{n} \mathbbold{1}_i \,\Big|\, P_1, \dots, P_n\Big]\Big] \quad \text{law of total expectation} \nonumber \\
&= \E \prod_{i=1}^{n} P_i \quad \text{by \eqref{eq:conditional-product}} \nonumber \\
&= \prod_{i=1}^{n} \E P_i \quad \text{as } \{P_i\} \text{ are i.i.d.} \nonumber \\
&= \mu^n \quad \text{as } \mu = \E P_i , \label{eq:mu-power}
\end{align}
Therefore,
\begin{equation}\label{eq:expected-accepted-sum}
A(B) = \sum_{n=1}^{B} \E \prod_{i=1}^{n} \mathbbold{1}_i = \sum_{n=1}^{B} \mu^n .
\end{equation}

Each drafting round also commits one more token, which is either the correction or the bonus token, so the expected number of generated tokens per drafting round is
\begin{equation}\label{eq:tokens-per-round-sum}
1 + A(B) = 1 + \sum_{n=1}^{B} \mu^n = \sum_{n=0}^{B} \mu^n .
\end{equation}

\end{proof}

\section{Rejection Sampling for Multi-Step Diffusion Drafting}
\begin{codextheorembox}
\begin{proposition}\label{prop:distribution-preserving}
Let $Q : \Omega \to \Delta(\vocab)$ be a random vector representing the final probability vector with which the draft token is sampled from.
$Q$ is of unknown distribution because the diffusion trajectory that produced the drafter probability vector is unknown. \\
Let $X : \Omega \to \vocab$ be a random variable that represents the outcome of performing rejection sampling with a fixed target distribution $\pt$ and the proposal distribution $Q$. Then
\begin{equation}
\label{eq:rs-marginal}
\mathbb{P}(X = v) = \pt(v) \quad \forall v \in \vocab
\end{equation}
\end{proposition}
\end{codextheorembox}

\begin{proof}[Proof of Proposition~\ref{prop:distribution-preserving}]
Based on rejection sampling theory~\citep{leviathan2023fast,chen2023accelerating}, conditioned on the actual probability $Q$ with which the drafted token is sampled from, we know that 
\begin{equation}
\label{eq:rs-conditional}
\mathbb{E}(\mathbbold{1}_{X = v} \mid Q) = \mathbb{P}(X = v \mid Q) = \pt(v) \quad \forall v \in \vocab.
\end{equation}

So by the law of total expectation, we have
\begin{equation*}
\mathbb{P}(X = v) = \mathbb{E}(\mathbbold{1}_{X = v}) = \mathbb{E}_q\!\left[\mathbb{E}(\mathbbold{1}_{X = v} \mid Q = q)\right]
= \mathbb{E}_q [\pt(v)] = \pt(v) \quad \forall v \in \vocab
\end{equation*}
Thus \eqref{eq:rs-marginal} holds.
\end{proof}

\section{Modified Cactus}
\label{app:codex-cactus}
\label{sec:other-divergence-choices}
Given a prefix $\mathbf{x}_{<t}$ and a drafted token $y$ with drafter probability $q(y\mid \mathbf{x}_{<t})$ and verifier probability $p(y\mid \mathbf{x}_{<t})$,
Cactus~\citep{hao2026cactus} creates a modified target probability distribution $\pt$ that depends on $y$ that maximizes $\pt(y \mid \mathbf{x}_{<t})$ subject to the constraint $D_{\mathrm{KL}}(\pt(\cdot\mid \mathbf{x}_{<t})\,\|\,p(\cdot\mid \mathbf{x}_{<t}))\leq\delta$. 
Specifically, it sets $\pt(y\mid \mathbf{x}_{<t})=\gamma$, where $\gamma$ is the solution to \eqref{eq:codex-cactus-kl}, and rescales the assigned probabilities for the rest of the tokens so that the overall vector remains a probability distribution, i.e., $\pt(v\mid \mathbf{x}_{<t})=\frac{1-\gamma}{1-p(y\mid \mathbf{x}_{<t})}\,p(v\mid \mathbf{x}_{<t})$ for $v\neq y$. For $0<p(y\mid \mathbf{x}_{<t})<1$,  reduces to
\begin{equation}
\label{eq:codex-cactus-kl}
\Phi(\gamma):=\gamma\log\frac{\gamma}{p(y\mid \mathbf{x}_{<t})}+(1-\gamma)\log\frac{1-\gamma}{1-p(y\mid \mathbf{x}_{<t})}\leq\delta.
\end{equation}
Cactus approximates $\Phi$ by its second-order Taylor expansion at $\gamma_0=p(y\mid \mathbf{x}_{<t})$,
\begin{equation}
\label{eq:codex-cactus-taylor}
\Phi(\gamma)\approx\Phi(\gamma_0)+\Phi'(\gamma_0)(\gamma-\gamma_0)+\frac{\Phi''(\gamma_0)}{2}(\gamma-\gamma_0)^2,
\end{equation}
and solves $\Phi(\gamma)=\delta$ under this approximation, which gives $\gamma=\min\{p(y\mid \mathbf{x}_{<t})+\sqrt{2\delta p(y\mid \mathbf{x}_{<t})(1-p(y\mid \mathbf{x}_{<t}))},\,1\}$. Cactus then accepts $y$ with probability $\min\{1,\gamma/q(y\mid \mathbf{x}_{<t})\}$.

However, the Taylor approximation is only justified when the desired modified probability for the drafted token is close to $\gamma_0=p(y\mid \mathbf{x}_{<t})$.

In order to support larger leniency $\delta$ values, when the verifier probability $\gamma_0=p(y\mid \mathbf{x}_{<t})$ is small and the drafter probability $q(y\mid \mathbf{x}_{<t})$ is large, the solution $\gamma^*$ needs to be much larger than $\gamma_0$ to become comparable to the drafter probability, which will break the Taylor approximation.

Therefore, in our experiments we use Newton's method to solve for the optimal $\gamma^*$ exactly for any given values of $\delta$, $p(y\mid \mathbf{x}_{<t})$, and $q(y\mid \mathbf{x}_{<t})$. On the GPU, the Newton steps can be run in parallel for the entire drafting block, so the runtime overhead is minimal.

\section{Experimental Details}
\label{app:experimental-details}
\codex{\label{app:experimental-config-degtail}}
All experiments were run on a single H200 GPU.

\begin{samepage}
% Source: drive/vgdllm-runs/drafter_verifier_gap_table.csv.
% Preserve the supplied rounding, including gaps computed before rounding accuracies.
\begin{table}[t]
\centering
\small
\setlength{\tabcolsep}{5pt}
\renewcommand{\arraystretch}{1.13}
\begin{tabular}{@{}p{3.45cm}lrrrr@{}}
\hline
Model pair & Benchmark & $n$ & Drafter (\%) & Verifier (\%) & Gap (pp)\\
\hline
\raggedright Dream 7B\newline$\to$ Qwen2.5 32B & GSM8K & 500 & $79.8\pm1.8$ & $95.0\pm1.0$ & $+15.2$\\
 & HumanEval & 164 & $57.3\pm3.9$ & $84.8\pm2.8$ & $+27.4$\\
 & GPQA Diamond & 198 & $22.7\pm3.0$ & $45.5\pm3.5$ & $+22.7$\\
\hline
\raggedright SDAR 8B\newline$\to$ Qwen3 8B & GSM8K & 500 & $85.0\pm1.6$ & $92.2\pm1.2$ & $+7.2$\\
 & HumanEval & 164 & $39.6\pm3.8$ & $85.4\pm2.8$ & $+45.7$\\
 & GPQA Diamond & 198 & $25.3\pm3.1$ & $43.4\pm3.5$ & $+18.2$\\
\hline
\raggedright DiffusionGemma 26B-A4B\newline$\to$ Gemma-4 31B & GSM8K & 500 & $96.4\pm0.8$ & $96.6\pm0.8$ & $+0.2$\\
 & HumanEval & 164 & $91.5\pm2.2$ & $93.3\pm2.0$ & $+1.8$\\
 & GPQA Diamond & 198 & $64.1\pm3.4$ & $73.2\pm3.1$ & $+9.1$\\
\hline
\end{tabular}
\caption{Baseline task accuracy with one standard error. The gap is verifier minus drafter accuracy in percentage points.}
\label{tab:drafter-verifier-gaps}
\end{table}

\subsection{Benchmark Selection}
\label{app:benchmark-selection}
\end{samepage}

The benchmarks use a random set of 500 GSM8K questions, all GPQA Diamond and all HumanEval problems. 
GSM8K and GPQA score the extracted final answer, while HumanEval reports accuracy from execution of the generated code against its tests using pass@1. We generate exactly one completion for all problems.

Table~\ref{tab:drafter-verifier-gaps} reports baseline accuracies used to select benchmarks. Dream uses 512 denoising steps on GSM8K and HumanEval and 1024 on GPQA. DiffusionGemma uses its native sampler with an upper limit of 48 diffusion steps. SDAR uses a confidence threshold of $\tau=0.9$. These settings assess drafter quality with high diffusion effort. We focus on settings where the drafter has useful task accuracy and the verifier offers room for improvement. Dream and SDAR meet these criteria on GSM8K and HumanEval, but perform worse than random on GPQA. DiffusionGemma has almost the same accuracy as the verifier model than on GSM8K and Humaneval so it is only run on GPQA Diamond. 

\subsection{Models, Data, and Sampling}
\label{app:run-settings}

Table~\ref{tab:model-checkpoints} shows the Huggingface checkpoints used for the experiments.
All models, both drafters and verifiers, are run without thinking enabled. All models have a generation limit of 2048 for GPQA Diamond and 512 on GSM8K and HumanEval.

\begin{table}[!htbp]
\centering
\small
\setlength{\tabcolsep}{5pt}
\renewcommand{\arraystretch}{1.2}
\begin{tabular}{@{}p{0.20\linewidth}p{0.52\linewidth}p{\dimexpr0.28\linewidth-4\tabcolsep\relax}@{}}
\hline
\textbf{Model / role} & \textbf{Hugging Face checkpoint} & \textbf{Reference}\\
\hline
\raggedright Dream drafter & \href{https://huggingface.co/Dream-org/Dream-v0-Instruct-7B}{\nolinkurl{Dream-org/Dream-v0-Instruct-7B}} & \raggedright\citet{ye2025dream}\tabularnewline
\raggedright Qwen2.5 verifier & \href{https://huggingface.co/Qwen/Qwen2.5-32B-Instruct}{\nolinkurl{Qwen/Qwen2.5-32B-Instruct}} & \raggedright\citet{qwen2024qwen25release}\tabularnewline
\raggedright DiffusionGemma drafter & \href{https://huggingface.co/google/diffusiongemma-26B-A4B-it}{\nolinkurl{google/diffusiongemma-26B-A4B-it}} & \raggedright\citet{diffusiongemma2026}\tabularnewline
\raggedright Gemma verifier & \href{https://huggingface.co/google/gemma-4-31B-it}{\nolinkurl{google/gemma-4-31B-it}} & \raggedright\citet{gemmateam2026gemma4}\tabularnewline
\raggedright Qwen3 verifier / Uno & \href{https://huggingface.co/s-sahoo/uno-qwen3-8B}{\nolinkurl{s-sahoo/uno-qwen3-8B}} & \raggedright\citet{sahoo2026unlocking}\tabularnewline
\raggedright SDAR 8B drafter & \href{https://huggingface.co/JetLM/SDAR-8B-Chat-b16}{\nolinkurl{JetLM/SDAR-8B-Chat-b16}} & \raggedright\citet{cheng2026sdar}\tabularnewline
\raggedright SDAR 1.7B drafter & \href{https://huggingface.co/JetLM/SDAR-1.7B-Chat-b16}{\nolinkurl{JetLM/SDAR-1.7B-Chat-b16}} & \raggedright\citet{cheng2026sdar}\tabularnewline
\raggedright DFlash drafter & \href{https://huggingface.co/z-lab/Qwen3-8B-DFlash-b16}{\nolinkurl{z-lab/Qwen3-8B-DFlash-b16}} & \raggedright\citet{chen2026dflash}\tabularnewline
\raggedright DSpark drafter & \href{https://huggingface.co/deepseek-ai/dspark_qwen3_8b_block7}{\nolinkurl{deepseek-ai/dspark_qwen3_8b_block7}} & \raggedright\citet{cheng2026dspark}\tabularnewline
\hline
\end{tabular}
\caption{Hugging Face checkpoints used in the experiments.}
\label{tab:model-checkpoints}
\end{table}

The Dream GSM8K experiment uses the model's chat template without an additional instruction prefix. On GSM8K, Gemma and the Qwen3 comparisons prepend the instruction to put the final answer inside \texttt{\textbackslash boxed\{\}}. Within each model-pair comparison, the AR baseline uses the corresponding prompt format.

\begin{table}[t]

\centering
\small
\setlength{\tabcolsep}{6pt}
\renewcommand{\arraystretch}{1.15}
\begin{tabular}{@{}lccc@{}}
\hline
Sampler & Temperature & Top-$p$ & Top-$k$\\
\hline
Dream draft / standalone & 0.2 & 0.95 & --\\
Qwen2.5-32B verifier & 0.7 & 0.8 & 20\\
DiffusionGemma draft / standalone & $0.8\to0.4$ & 0.95 & --\\
Gemma-4-31B verifier & 1.0 & 0.95 & 64\\
SDAR draft & 0.6 & 0.95 & 20\\
Uno draft & 1.0 & 0.95 & 50\\
DFlash draft & Greedy & -- & --\\
DSpark draft & 1.0 & -- & --\\
Qwen3-8B verifier / AR & 1.0 & 0.95 & 50\\
\hline
\end{tabular}
\caption{Sampling settings for the reported runs. A dash indicates no additional truncation recorded for that setting. DiffusionGemma anneals temperature during denoising and uses 0.4 for a single denoising step.}
\label{tab:sampling-settings}

\end{table}

Dream 7B uses entropy-based unmasking with deterministic position selection. Each drafting round denoises the remaining generation canvas.

DiffusionGemma uses a block size of 256. When the DiffusionGemma is run as a standalone model, it uses an entropy bound sampler which terminates the diffusion when a confidence threshold is reached, so the number of diffusion steps is variable, but is bounded by an upper limit. In our experiments, the values of 1, 2, 4, 8, 16, 32, and 48 are used as the upper bound. When DiffusionGemma is run as a drafter, we modify it to run a fixed number of diffusion steps ($K$) instead of using the original entropy bound sampler, because DiffusionGemma has both a stability threshold and a confidence threshold to determine when to terminate the diffusion, along with the entropy bound sampler threshold, which makes it not straightforward to control the diffusion effort with the original sampling approach.

\begin{samepage}
    
\begin{table}[t]
\centering
\small
\setlength{\tabcolsep}{3pt}
\renewcommand{\arraystretch}{1.16}
\begin{tabular}{@{}p{3.65cm}p{1.65cm}p{2.95cm}p{1.5cm}rr@{}}
\hline
Model / pair & Dataset & Diffusion effort & Acceptance & Block & \shortstack{Token\\limit}\\
\hline
\raggedright Dream 7B + Qwen2.5 32B & GSM8K & $K\in\{1,2,3,4\}$ & Cactus & $\leq512$ & 512\\
 & HumanEval & $K\in\{1,2,3,4\}$ & Cactus & $\leq512$ & 512\\
\raggedright DiffusionGemma 26B-A4B + Gemma-4 31B & GPQA & $K\in\{1,2,4,8\}$ & Cactus & 256 & 2048\\
 & GPQA & $K\in\{1,2,4,8\}$ & Linear & 256 & 2048\\
SDAR 8B & GSM8K,\newline HumanEval & $\tau=0.3,0.5,0.7,0.9$\newline or fixed $K=1$ & Linear & 16 & 512\\
SDAR 1.7B & GSM8K,\newline HumanEval & $\tau=0.3,0.5,0.7,0.9$\newline or fixed $K=1$ & Linear & 16 & 512\\
Uno & GSM8K,\newline HumanEval & $K=1$ & Linear & 16 & 512\\
DFlash & GSM8K,\newline HumanEval & One draft pass & Linear & 16 & 512\\
DSpark & GSM8K,\newline HumanEval & Fixed sampler & Linear & 7 & 512\\
\hline
\end{tabular}
\caption{Main experimental configurations. The final five rows share the Qwen3-8B verifier. Block lengths are the configured values, distinct from full-completion limits. DFlash's 16-position block includes one anchor and 15 draft positions. SDAR's confidence threshold $\tau$ controls denoising effort adaptively.}
\label{tab:main-experiment-settings}
\end{table}

\label{app:drafter-controls}SDAR commits tokens whose confidence exceeds $\tau$, so its number of denoising passes varies with the draft. We sweep $\tau\in\{0.3,0.5,0.7,0.9\}$ for both model sizes. We also run another configuration where the SDAR drafter runs exactly one diffusion step instead of relying on the confidence threshold for termination. All Qwen3 comparisons use $\lambda\in\{0,0.2,0.4,0.6,0.8,1\}$ with the same verifier model and sampling settings. Table~\ref{tab:main-experiment-settings} gives block sizes and completion limits. DSpark retains its native confidence threshold of zero.
\end{samepage}

Figure~\ref{fig:sdar-uno-frontier} uses one-pass Uno drafting. Appendix~\ref{app:uno-multistep} reports the additional $K=2,4$ experiments that support this choice. For DFlash and DSpark, we vary acceptance leniency while keeping their released drafting procedures.

\begin{samepage}

\subsection{Sweep Grids and Measurement Protocol}
\label{app:sweeps-and-timing}

Table~\ref{tab:cactus-grid} lists the Cactus budgets for the main results, along with the diffusion effort $K$. 

On the other hand, for the linear acceptance modifier, $\lambda\in\{0,0.1,\ldots,1\}$ is used, along with
$K\in\{1,2,3,4\}$ for the Dream pair on GSM8K,
$K\in\{1,2,3\}$ for the Dream pair on HumanEval,
$K\in\{1,2,4,8\}$ for DiffusionGemma pair on GPQA,
and $K\in\{1,2,4,6,10\}$ for DiffusionGemma pair on HumanEval.

\begin{table}[t]
\centering
\small
\setlength{\tabcolsep}{8pt}
\renewcommand{\arraystretch}{1.12}
\begin{tabular}{@{}lcl@{}}
\hline
Model pair / dataset & $K$ & Cactus budgets $\delta$\\
\hline
Dream--Qwen / GSM8K & 1 & $0,0.25,0.5,1,\ldots,8$\\
 & 2 & $0,0.25,0.5,1,\ldots,11$\\
 & 3 & $0,0.25,0.5,1,\ldots,13$\\
 & 4 & $0,0.25,0.5,1,\ldots,15$\\
DiffusionGemma--Gemma / GPQA & 1 & $0,0.25,0.5,1,\ldots,5$\\
 & 2 & $0,0.25,0.5,1,\ldots,6$\\
 & 4 & $0,0.25,0.5,1,\ldots,8,10$\\
 & 8 & $1,3,5,7,10,15$\\
\hline
Dream--Qwen / HumanEval & 1 & $0,0.25,0.5,1,\ldots,6$\\
 & 2 & $0,0.25,0.5,1,\ldots,12$\\
 & 3 & $0,0.25,0.5,1,\ldots,13$\\
 & 4 & $0,0.25,0.5,1,\ldots,15$\\
DiffusionGemma--Gemma / HumanEval & 1 & $0,0.25,0.5,1,\ldots,8$\\
 & 2 & $0,0.25,0.5,1,\ldots,12$\\
 & 4 & $0,0.25,0.5,1,\ldots,16$\\
 & 6 & $0,0.25,0.5,1,\ldots,15$\\
 & 10 & $0.25,0.5,1,\ldots,15,20$\\
\hline
\end{tabular}
\caption{Evaluated Cactus budgets at each diffusion effort.}
\label{tab:cactus-grid}
\end{table}

\end{samepage}

\subsection{Simple Throughput Model Calibration}
\label{app:cost-calibration}

The calibration for Figure~\ref{fig:gsm8k_scaling_law} uses $K\in\{1,2,3,4\}$ and runs 50 random questions from the GSM8K dataset. For each $K$, we compute the average TVD between the drafter and verifier distributions. We then substitute them into \eqref{eq:theoretical_speed_quality} and derive the theoretical throughput and quality at different leniency values. The block size $B$ is set to 512 in \eqref{eq:throughput_model}.

Based on empirical profiling of the Dream-7B drafter and Qwen2.5-32B verifier, we estimate $d_K$ to be $0.566,0.517,0.424,0.360$ for $K=1,2,3,4$ respectively. These distances imply lossless acceptance rates \(1-d_K\) of 0.434, 0.483, 0.576, and 0.640, which determine the starting points of the lines in Figure~\ref{fig:tvd-theoretical-acceptance}. Figure~\ref{fig:gsm8k-profiling-costs} shows the goodness of fit for the linear regression plot and the constant time for calling the verifier.

\begin{figure}[!htbp]
\centering
\includegraphics[width=\linewidth]{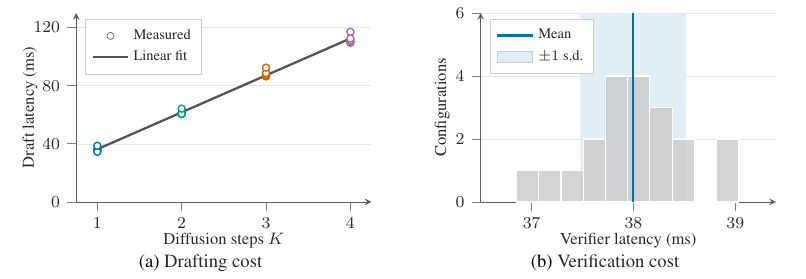}
\captionsetup{font=small,skip=4pt}
\caption{Per-round latency measurements supporting Figure~\ref{fig:gsm8k_scaling_law}. Left: drafting costs and their linear fit in $K$. Right: the distribution of verifier costs across configurations, with its mean and one-standard-deviation band.}
\label{fig:gsm8k-profiling-costs}
\end{figure}

\section{Additional Results}
\label{app:additional-downstream}

\subsection{Linear Acceptance}
\label{app:linear-acceptance-results}

Figures~\ref{fig:appendix-dream-gsm8k} and~\ref{fig:appendix-gemma-gpqa} show the linear-acceptance counterparts of Figure~\ref{fig:joint-frontier}, alongside Cactus for comparison.

\begin{figure}[!htbp]
\centering
\captionsetup[subfigure]{font=small,justification=centering,skip=2pt,hypcap=true}
\begin{subfigure}[t]{0.49\linewidth}
\centering
\includegraphics[page=1,width=\linewidth]{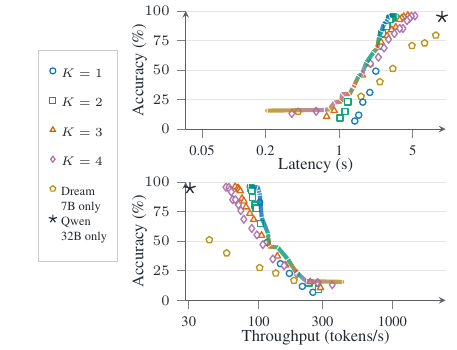}
\caption{Cactus}
\label{fig:gsm8k_time_cactus}
\end{subfigure}\hfill
\begin{subfigure}[t]{0.49\linewidth}
\centering
\includegraphics[page=2,width=\linewidth]{figures/appendix_dream_gsm8k_panels.pdf}
\caption{Linear acceptance}
\label{fig:gsm8k_time_cactus_with_ar}
\end{subfigure}
\captionsetup{font=small,skip=4pt}
\caption{Latency and throughput on GSM8K with the Dream-7B drafter and Qwen2.5-32B-Instruct verifier}
\label{fig:appendix-dream-gsm8k}
\end{figure}

\begin{figure}[!htbp]
\centering
\captionsetup[subfigure]{font=small,justification=centering,skip=2pt,hypcap=true}
\begin{subfigure}[t]{0.49\linewidth}
\centering
\includegraphics[page=1,width=\linewidth]{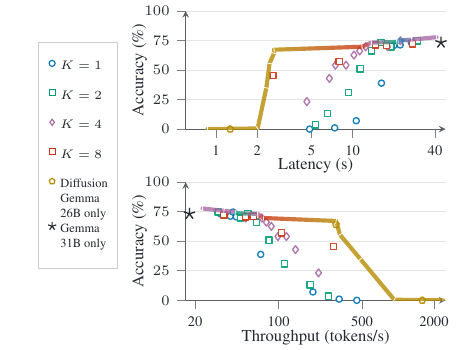}
\caption{Cactus}
\label{fig:gpqa_gemma_time_cactus}
\end{subfigure}\hfill
\begin{subfigure}[t]{0.49\linewidth}
\centering
\includegraphics[page=2,width=\linewidth]{figures/appendix_gemma_gpqa_panels.pdf}
\caption{Linear acceptance}
\label{fig:gpqa_gemma_time_convex}
\end{subfigure}
\captionsetup{font=small,skip=4pt}
\caption{Latency and throughput on GPQA Diamond with the DiffusionGemma 26B drafter and Gemma4 31B verifier.}
\label{fig:appendix-gemma-gpqa}
\end{figure}

\clearpage
\subsection{Drafting Block Size Ablation}
\label{app:dream-block-size}

We vary the drafting block size for Dream-7B with Qwen2.5-32B-Instruct on the 500 GSM8K questions, using linear acceptance and a 512-token generation limit. For $B\in\{32,64,128\}$, we sweep $K\in\{1,2,3,4\}$ and $\lambda\in\{0,0.2,0.4,0.6,0.8,1\}$. We compare these runs with the remaining-canvas configuration, denoted $B\leq512$ which means drafting over the remaining generation canvas. We compute a separate empirical Pareto frontier over all measured settings within each block-size configuration. Figure~\ref{fig:dream-block-size} shows that changing block size affects latency and throughput, while preserving the main observation of Section~\ref{sec:joint-exploration-results}. Different diffusion efforts contribute to different regions of the frontier.
\begin{figure}[!htbp]
\centering
\captionsetup[subfigure]{font=small,justification=centering,skip=2pt,hypcap=true}
\begin{subfigure}[t]{\linewidth}
\centering
\codex{\includegraphics[width=\linewidth]{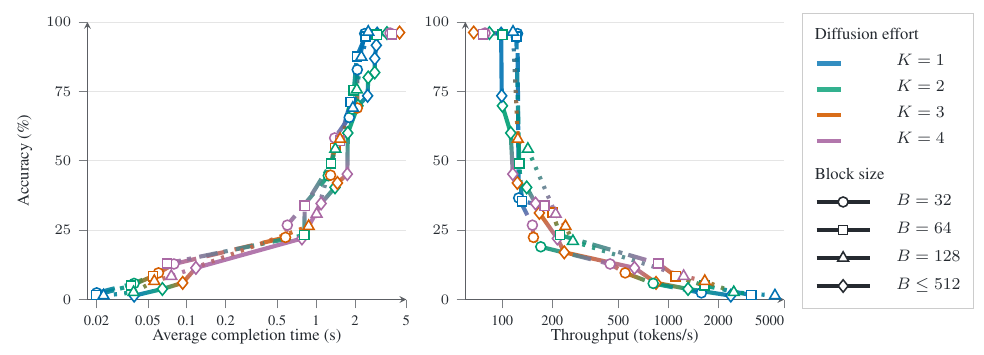}}
\caption{Absolute latency and throughput}
\label{fig:dream-block-size-absolute}
\end{subfigure}
\par\medskip
\begin{subfigure}[t]{\linewidth}
\centering
\codex{\includegraphics[width=\linewidth]{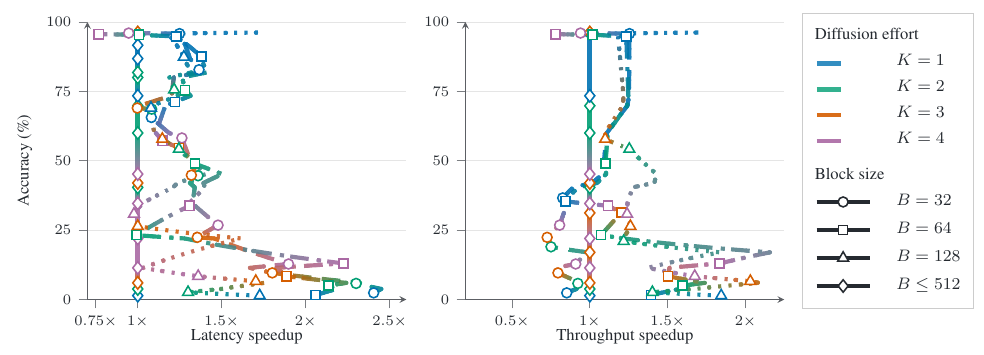}}
\caption{Speedup relative to the $B\leq512$ frontier}
\label{fig:dream-block-size-relative}
\end{subfigure}
\captionsetup{font=small,skip=4pt}
\caption{Block-size ablation on GSM8K with Dream--Qwen.}
\label{fig:dream-block-size}
\end{figure}

\clearpage
\subsection{HumanEval}
\label{app:humaneval-results}

Figures~\ref{fig:appendix-dream-humaneval} and~\ref{fig:appendix-gemma-humaneval} extend the Dream--Qwen and DiffusionGemma--Gemma comparisons to HumanEval, using both acceptance rules.

\begin{figure}[!htbp]
\centering
\captionsetup[subfigure]{font=small,justification=centering,skip=2pt,hypcap=true}
\begin{subfigure}[t]{0.49\linewidth}
\centering
\includegraphics[page=1,width=\linewidth]{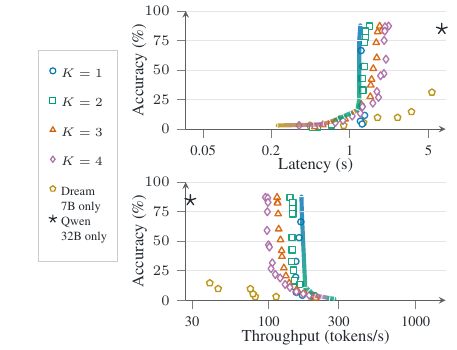}
\caption{Cactus}
\label{fig:humaneval_time_cactus}
\end{subfigure}\hfill
\begin{subfigure}[t]{0.49\linewidth}
\centering
\includegraphics[page=2,width=\linewidth]{figures/appendix_dream_humaneval_panels.pdf}
\caption{Linear acceptance}
\label{fig:humaneval_time_cactus_with_ar}
\end{subfigure}
\captionsetup{font=small,skip=4pt}
\caption{Latency and throughput on HumanEval with the Dream 7B drafter and Qwen2.5 32B verifier.}
\label{fig:appendix-dream-humaneval}
\end{figure}

\begin{figure}[!htbp]
\centering
\captionsetup[subfigure]{font=small,justification=centering,skip=2pt,hypcap=true}
\begin{subfigure}[t]{0.49\linewidth}
\centering
\includegraphics[page=1,width=\linewidth]{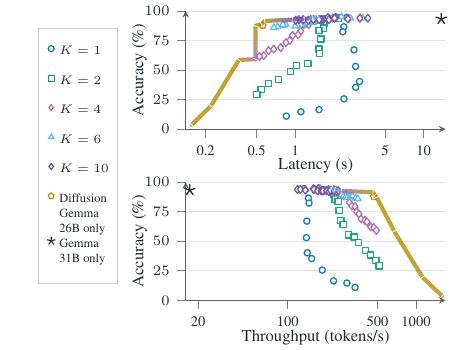}
\caption{Cactus}
\label{fig:humaneval_gemma_time_cactus}
\end{subfigure}\hfill
\begin{subfigure}[t]{0.49\linewidth}
\centering
\includegraphics[page=2,width=\linewidth]{figures/appendix_gemma_humaneval_panels.pdf}
\caption{Linear acceptance}
\label{fig:humaneval_gemma_time_convex}
\end{subfigure}
\captionsetup{font=small,skip=4pt}
\caption{Latency and throughput on HumanEval with the DiffusionGemma 26B drafter and Gemma4 31B verifier.}
\label{fig:appendix-gemma-humaneval}
\end{figure}

\clearpage
\subsection{Shared-Verifier Comparison}
\label{app:shared-verifier-sweeps}

\begin{figure}[!htbp]
\centering
\codex{\includegraphics[width=\linewidth]{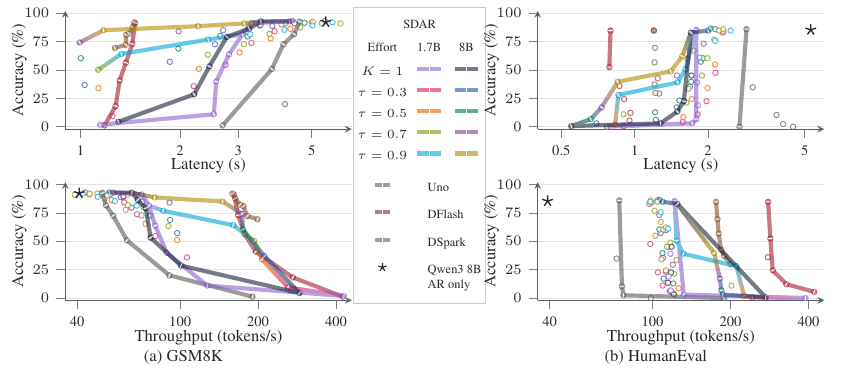}}
\captionsetup{font=small,skip=4pt}
\caption{Complete sweeps for Figure~\ref{fig:sdar-uno-frontier}. Colors distinguish SDAR's effort settings within each model size. Hollow markers show measured configurations, including dominated points. Thick curves show the joint and fixed-$K=1$ frontiers, using the same definitions as the main figure.}
\label{fig:appendix-sdar-comparison}
\end{figure}

\subsection{Uno at Multiple Diffusion Steps}
\label{app:uno-multistep}

We evaluate Uno with $K\in\{1,2,4\}$ and linear acceptance at $\lambda\in\{0,0.2,0.4,0.6,0.8,1\}$. We hold the block size at 16 and use the same problems, sampling settings, and 512-token limit as Figure~\ref{fig:sdar-uno-frontier}.

In Figure~\ref{fig:uno-multistep}, $K=1$ captures the entire latency frontier on both benchmarks and the entire throughput frontier on HumanEval. On GSM8K, one low-accuracy $K=2$ configuration also contributes to the measured throughput frontier, while $K=4$ contributes to neither frontier. Thus, additional steps do not provide the speed-quality frontier improvements seen with diffusion drafters that are trained to perform as standalone models. This pattern is consistent with Uno's one-step distillation objective, which trains it to recover a clean block from fully corrupted inputs~\citep{sahoo2026unlocking} following the base model's distribution.

\begin{figure}[!htbp]
\centering
\captionsetup[subfigure]{font=small,justification=centering,skip=2pt,hypcap=true}
\begin{subfigure}[t]{0.49\linewidth}
\centering
\includegraphics[page=1,width=\linewidth]{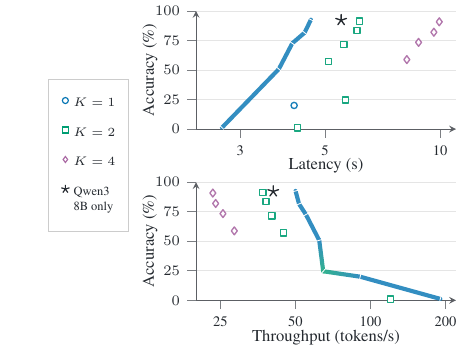}
\caption{GSM8K}
\label{fig:uno-multistep-gsm8k}
\end{subfigure}\hfill
\begin{subfigure}[t]{0.49\linewidth}
\centering
\includegraphics[page=2,width=\linewidth]{figures/appendix_uno_multistep_panels.pdf}
\caption{HumanEval}
\label{fig:uno-multistep-humaneval}
\end{subfigure}
\captionsetup{font=small,skip=4pt}
\caption{Uno with one, two, and four drafting steps. Colors indicate $K$, with leniency varied within each color. Thick lines connect nondominated measurements, and hollow markers show the remaining configurations. The GSM8K sweep lacks completed $K=4$ runs at $\lambda\in\{0.8,1\}$.}
\label{fig:uno-multistep}
\end{figure}

% Complete prompts and saved responses: data/qualitative_examples.json.
% Invisible listing delimiters apply emphasis only. No generated text is omitted.
\definecolor{qualink}{HTML}{29465B}
\definecolor{qualcorrect}{HTML}{287566}
\definecolor{qualerror}{HTML}{A34736}
\newtcolorbox{qualprompt}[1]{%
  breakable,colback=qualink!3!white,colframe=qualink,
  colbacktitle=qualink,coltitle=white,fonttitle=\small\bfseries,title={#1},
  title after break={#1\ (continued)},
  boxrule=0.6pt,arc=2pt,left=8pt,right=8pt,top=6pt,bottom=6pt,
  before skip=7pt,after skip=7pt,fontupper=\small}
\newtcolorbox{qualresponse}[2]{%
  breakable,colback=white,colframe=#1!45!white,
  colbacktitle=#1!9!white,coltitle=#1,fonttitle=\small\bfseries,
  title={#2},title after break={#2\ (continued)},
  boxrule=0.4pt,arc=1pt,left=7pt,right=7pt,top=4pt,bottom=4pt,
  toptitle=3pt,bottomtitle=3pt,before skip=6pt,after skip=6pt,
  fontupper=\small}
\lstdefinestyle{qualfull}{%
  basicstyle=\fontfamily{lmtt}\fontsize{8.5}{10.2}\selectfont,
  % Fixed columns preserve source spacing and let wrapped lines inherit indentation.
  columns=fixed,basewidth=4.46pt,keepspaces=false,showstringspaces=false,upquote=true,
  breaklines=true,breakatwhitespace=false,breakautoindent=true,breakindent=0pt,
  literate={-}{{\char45}}1 {*}{{\char42}}1,
  aboveskip=1pt,belowskip=1pt,
  moredelim=**[is][\color{qualcorrect}\bfseries]{@@GOOD@@}{@@END@@},
  moredelim=**[is][\color{qualerror}\bfseries]{@@ERROR@@}{@@END@@}}
\newcommand{\qualaxis}[1]{\par\addvspace{8pt}\noindent\textcolor{qualink}{\textbf{#1}}\par\nobreak}

\clearpage
\raggedbottom
\section{Qualitative Examples}
\label{app:qualitative-examples}

We present qualitative examples of speculative diffusion using Dream-7B with Qwen2.5-32B-Instruct under linear acceptance. These selected samples illustrate failure modes of high acceptance leniency. The original benchmark prompts and selected responses are displayed in full, including repetition, formatting markers, and saved end tokens. Colored bold text highlights correct calculations or code (green), errors or repetition (red).

\subsection{Math}
\label{app:qualitative-math}
\begin{qualprompt}{GSM8K 1116 \hfill Full benchmark prompt}
\begin{lstlisting}[style=qualfull]
Last year there were 50 students enrolled in a calligraphy class. This year, there was a 20% increase in enrollment. How many students are enrolled this year in the calligraphy class?
\end{lstlisting}
\end{qualprompt}
\noindent\textbf{Reference answer.} $50+0.20\times50=60$ students.
\qualaxis{Fixed effort $K=1$: increasing acceptance leniency}
% Full saved response: enrollment/dsteps1_r1.0
\begin{qualresponse}{qualcorrect}{$K=1,\ \lambda=0$ \hfill Correct}
\begin{lstlisting}[style=qualfull]
To determine the number of students enrolled in the calligraphy class this year, we need to calculate a 20% increase on last year's enrollment of 50 students.

First, we calculate 20% of 50:
\[
20\% \text{ of } 50 = \frac{20}{100} \times 50 = 0.2 \times 50 = 10
\]

Next, we add this increase to last year's enrollment:
\[
@@GOOD@@50 + 10 = 60@@END@@
\]

Thus, the number of students enrolled this year in the calligraphy class is @@GOOD@@\boxed{60}@@END@@.<|im_end|>
\end{lstlisting}
\end{qualresponse}

% Full saved response: enrollment/dsteps1_r0.6
\begin{qualresponse}{qualerror}{$K=1,\ \lambda=0.4$ \hfill Incorrect}
\begin{lstlisting}[style=qualfull]
To determine the number of students enrolled this year in the calligraphy class, we need to calculate a 20% increase on last year's enrollment of 50 students.

 0. First, we calculate 20% of 50 students:
\[ 0.20 \times 50 = @@ERROR@@00 = 10@@END@@ \]

@@ERROR@@Next. we,@@END@@ add this increase to the original number of students:
\[ @@GOOD@@50 + 10 = 60@@END@@ \]

Therefore, the @@ERROR@@number of of@@END@@ students enrolled this year in the calligraphy class is $@@ERROR@@\boxed{600}@@END@@$.<|im_end|>
\end{lstlisting}
\end{qualresponse}

\clearpage
% Full saved response: enrollment/dsteps1_r0.4
\begin{qualresponse}{qualerror}{$K=1,\ \lambda=0.6$ \hfill Incorrect}
\begin{lstlisting}[style=qualfull]
To determine the number of students enrolled this year in the calligraphy class, we need to calculate the increase in enrollment and last year's enrollment.

1. Calculate the increase in enrollment: 20% of last year's enrollment.
2. Add the increase to last year's enrollment to find this year's enrollment.

Let's the this step-by-step:

```python
 the number of students enrolled last year.
last_year_enrollment = 50

 the percentage in enrollment.
percentage_increase = 20

 the increase in enrollment.
increase = (percentage_increase / 100) * last_year_enrollment

 the year's enrollment.
this_year_enrollment = last_year_enrollment + increase

@@ERROR@@print_year_year_en)@@END@@
```
`
`
output
SyntaxError: invalid syntax
`
It looks like there was syntax error in the code. @@ERROR@@Let Let me correct the and calculate the number's enrollment this year year properly.@@END@@

```python
@@ERROR@@# number of students enrolled last last last last last last = = =5500000000

# percentage in in in in in in in in in in in in in in in in in in in in in in in in in in in in in in in in in in in in in in in in in in in0 in in in in in in in in0 in in in in in in in in in in in in in in in in in in in in in in in in in in in in in in in in in in in in in in in in in in in in in in in in in in in in in in in in in in in in in in in in in in in in in in in in in in in in in in in in in in in in in in in in in in in in in in in in in in in in in in in in in in in in in in in in in in in in in in in in in in in in in in in in in in in in in in in in in in in in in in in in in in in in in in in in in in in in in in in in in in in in in in in in in in in in in in in in in in in in in in in in in in in in in in in in in@@END@@
\end{lstlisting}
\end{qualresponse}

\clearpage
\qualaxis{Increasing diffusion effort at $\lambda=0.4$ and $\lambda=0.6$}
% Full saved response: enrollment/dsteps2_r0.6
\begin{qualresponse}{qualcorrect}{$K=2,\ \lambda=0.4$ \hfill Correct}
\begin{lstlisting}[style=qualfull]
To determine the number of students enrolled this year in the calligraphy class, we need to calculate a 20% increase on last year's enrollment of 50 students.

A 20% increase means that the number of students this year is 120% of last year's enrollment. 

To calculate 120% of 50, we can multiply 50 by 1.20 (since 120% is equivalent to 1.20).

Let's perform the calculation:

\[
@@GOOD@@50 \times 1.20 = 60@@END@@
\]

Therefore, the number of students enrolled this year in the calligraphy class is $@@GOOD@@\boxed{60}@@END@@$.
\end{lstlisting}
\end{qualresponse}

% Full saved response: enrollment/dsteps4_r0.4
\begin{qualresponse}{qualcorrect}{$K=4,\ \lambda=0.6$ \hfill Correct}
\begin{lstlisting}[style=qualfull]
To determine the number of students enrolled this year in the calligraphy class, we need to calculate the increase in enrollment and add it to last year's enrollment.

The increase in enrollment is 20% of last year's enrollment. Last year, there were 50 students enrolled.
20% of 50 can be calculated as follows: \(@@GOOD@@50 \times \frac{20}{100} = 50 \times 0.20 = 10@@END@@\).

So, the increase in enrollment is \(10\) students.

Therefore, this year's enrollment is: \(@@GOOD@@50 + 10 = \boxed{60 \ students}@@END@@\).
\end{lstlisting}
\end{qualresponse}

\clearpage
\subsection{Code}
\label{app:qualitative-code}
\begin{qualprompt}{HumanEval/152 \hfill Full benchmark prompt}
\begin{lstlisting}[style=qualfull]

def compare(game,guess):
    """I think we all remember that feeling when the result of some long-awaited
    event is finally known. The feelings and thoughts you have at that moment are
    definitely worth noting down and comparing.
    Your task is to determine if a person correctly guessed the results of a number of matches.
    You are given two arrays of scores and guesses of equal length, where each index shows a match. 
    Return an array of the same length denoting how far off each guess was. If they have guessed correctly,
    the value is 0, and if not, the value is the absolute difference between the guess and the score.
    
    
    example:

    compare([1,2,3,4,5,1],[1,2,3,4,2,-2]) -> [0,0,0,0,3,3]
    compare([0,5,0,0,0,4],[4,1,1,0,0,-2]) -> [4,4,1,0,0,6]
    """
\end{lstlisting}
\end{qualprompt}
\qualaxis{Fixed effort $K=1$: increasing acceptance leniency}
% Full saved response: code/dsteps1_r1.0
\begin{qualresponse}{qualcorrect}{$K=1,\ \lambda=0$ \hfill Pass}
\begin{lstlisting}[style=qualfull]
def compare(game, guess):
    """Compares the actual scores of matches with guessed scores and returns the absolute differences."""
    result = []
    @@GOOD@@for i in range(len(game)):@@END@@
        result.append(abs(game[i] - guess[i]))
    return result
\end{lstlisting}
\end{qualresponse}

% Full saved response: code/dsteps1_r0.4
\begin{qualresponse}{qualerror}{$K=1,\ \lambda=0.6$ \hfill Fail}
\begin{lstlisting}[style=qualfull]
@@ERROR@@```python thedef
compare(game, guess):
         
for
i in range(len(len
)):@@END@@

        @@ERROR@@diff = abs[i]game[i][i[i[i[i[i        ] - guess[i[i[i[i[i[i[i[i[i[i[i[i[i[i[i[i[i[i[i[i[i[i[i[i[i[i[i[i[i[i[i[i[i[i[i[i[i[i[i[i[i[i[i[i[i[i[i[i[i[i[i[i[i[i[i[i[i[i[i[i[i[i[i[i[i[i[i[i[i[i[i[i[i[i[i[i[i[i[i[i[i[i[i[i[i[i[i[i[i[i[i[i[i[i[i[i[i[i[i[i[i[i[i[i[i[i[i[i[i[i[i[i[i[i[i[i[i[i[i[i[i[i[i[i[i[i[i[i[i[i[i[i[i[i[i[i[i[i[i@@END@@
\end{lstlisting}
\end{qualresponse}

\clearpage
\qualaxis{Fixed leniency $\lambda=0.4$: increasing diffusion effort}
% Full saved response: code/dsteps1_r0.6
\begin{qualresponse}{qualerror}{$K=1,\ \lambda=0.4$ \hfill Fail}
\begin{lstlisting}[style=qualfull]
```python
def compare(game, guess):
    """I think we all remember that feeling when the result of some long-awaited
    event is finally known. The feelings and thoughts you have at that moment are
    definitely worth noting down and comparing.
    Your task is to determine if a person correctly guessed the results of a number of @@ERROR@@matches matches@@END@@    You are given two arrays of scores and guesses of equal length, where each index shows a match. 
    Return an array of the same length denoting how far off each guess was. If they have guessed correctly,
    the value is 0, and if not, the value is the absolute difference between the guess and the score.
    """
    result = []
    for i in range(@@ERROR@@len(len(game))@@END@@):
        result.append(abs(game[i] - guess[i]))
    return result
\end{lstlisting}
\end{qualresponse}

% Full saved response: code/dsteps3_r0.6
\begin{qualresponse}{qualcorrect}{$K=3,\ \lambda=0.4$ \hfill Pass}
\begin{lstlisting}[style=qualfull]
def compare(game, guess):
    """I think we all remember that feeling when the result of some event is finally known. The feelings and thoughts you have at that moment are definitely worth noting down and comparing.
    Your task is to determine if a person correctly guessed the results of a number of matches.
    You are given two arrays of scores and guesses of equal length, where each index shows a match. 
    Return an array of the same length denoting how far off each guess was. If they have guessed correctly,
    the value is 0, and if not the value is the absolute difference between the @@ERROR@@guess guess@@END@@ and the score.
    
    
    example:

    compare([1,2,3,4,5,1],[1,2,3,4,2,-2]) -> [0,0,0,0,3,3]
    compare([0,5,0,0,0,4],[4,1,1,0,0,-2]) -> [4,4,1,0,0,6]
    """
    @@GOOD@@return [abs(game[i] - guess[i]) for i in range(len(game))]@@END@@
\end{lstlisting}
\end{qualresponse}

\end{document}